\documentclass{article}

\usepackage[nonatbib,final]{neurips_2024}

\usepackage[utf8]{inputenc} 
\usepackage[T1]{fontenc}    
\usepackage{hyperref}       
\usepackage{url}            
\usepackage{booktabs}       
\usepackage{amsfonts}       
\usepackage{nicefrac}       
\usepackage{microtype}      
\usepackage{xcolor}         
\usepackage[numbers]{natbib}
\usepackage{amsmath}
\usepackage{bm}
\usepackage{tikz}
\usepackage{wrapfig}
\usepackage{amsthm}
\usepackage{listings}
\usepackage{multirow}
\usepackage{algorithm}
\usepackage{float}
\usepackage{algpseudocode}
\usepackage[frozencache=true,cachedir=minted-cache]{minted}
\usepackage{etoolbox}
\makeatletter
\patchcmd{\minted@colorbg}{\medskip}{}{}{}
\patchcmd{\endminted@colorbg}{\medskip}{}{}{}
\makeatother

\newcommand{\A}{\mathbf{A}}
\newcommand{\B}{\mathbf{B}}
\newcommand{\C}{\mathbf{C}}
\newcommand{\rr}{\mathbf{r}}

\newcommand{\M}{\mathbf{M}}
\newcommand{\OO}{\mathcal{O}}
\newcommand{\I}{\mathbf{I}}
\newcommand{\bb}{\mathbf{b}}
\newcommand{\zz}{\mathbf{z}}
\newcommand{\cc}{\mathbf{c}}
\newcommand{\sss}{\mathbf{s}}
\newcommand{\ttt}{\mathbf{t}}

\newcommand{\TT}{\mathbf{\Theta}}
\newcommand{\mm}{\bm{\mu}}
\newcommand{\SSS}{\mathbf{\Sigma}}

\newcommand{\SSU}{\mathbf{\Sigma_\uu}}
\newcommand{\SSUI}{\mathbf{\Sigma_\uu^{-1}}}
\newcommand{\SSV}{\mathbf{\Sigma_\vv}}
\newcommand{\SSVI}{\mathbf{\Sigma_\vv^{-1}}}
\newcommand{\ssV}{\mathbf{\tau_\vv}}

\newcommand{\SSY}{\mathbf{\Sigma_\yy}}
\newcommand{\SSYI}{\mathbf{\Sigma_\yy^{-1}}}
\newcommand{\ssY}{\mathbf{\tau_\yy}}

\newcommand{\uu}{\mathbf{u}}
\newcommand{\Q}{\mathbf{Q}}
\newcommand{\LL}{\mathbf{\Lambda}}

\newcommand{\yy}{\mathbf{y}}

\newcommand{\xx}{\mathbf{x}}
\newcommand{\vv}{\mathbf{v}}
\newcommand{\zero}{\mathbf{0}}
\newcommand{\LLL}{\mathbf{L}}
\newcommand{\TTT}{\mathbf{T}}
\newcommand{\EE}{\mathbf{E}}
\newcommand{\FF}{\mathbf{F}}
\newcommand{\GG}{\mathbf{G}}

\newcommand{\HC}{\text{HalfCauchy}}
\newcommand{\LKJ}{\text{LKJCholesky}}

\tikzset{elegant/.style={smooth, black, thick, samples=101}}
\tikzstyle{round}=[circle ,text centered,draw=black]
\tikzstyle{arrow} = [-,>=stealth, thick]
\tikzstyle{rct}=[rectangle,draw,thin,fill=white]
\usetikzlibrary{shapes.geometric,arrows,decorations.pathmorphing,backgrounds,positioning,fit,matrix,calc}

\newcommand{\ds}[1]{
}
\newcommand{\blue}[1]{\textcolor{black}{#1}}

\newcommand{\camread}[1]{\textcolor{black}{#1}}

\newtheorem{theorem}{Theorem}
\newtheorem{lemma}{Lemma}

\title{Hamiltonian Monte Carlo Inference of Marginalized Linear Mixed-Effects Models}

\author{%
  Jinlin Lai,\quad Justin Domke,\quad Daniel Sheldon \\
  Manning College of Information and Computer Sciences\\ University of Massachusetts Amherst \\
  \texttt{\{jinlinlai,domke,sheldon\}}@cs.umass.edu \\
}

\begin{document}

\maketitle

\begin{abstract}
Bayesian reasoning in linear mixed-effects models (LMMs) is challenging and often requires advanced sampling techniques like Markov chain Monte Carlo (MCMC).
A common approach is to write the model in a probabilistic programming language and then sample via Hamiltonian Monte Carlo (HMC).
However, there are many ways a user can transform a model that make \camread{inference} \ds{"HMC" $\to$ "inference"?} more or less efficient.
In particular, marginalizing some variables can greatly improve \camread{inference} but is difficult for users to do manually.
We develop an algorithm to easily marginalize random effects in LMMs.
A naive approach introduces cubic time operations within \camread{an inference algorithm like} HMC, but we reduce the running time to linear using fast linear algebra techniques.
We show that marginalization is always beneficial when applicable and highlight improvements in various models, especially ones from cognitive sciences\footnote{The code is available at \url{https://github.com/lll6924/hamiltonian_lme.git}.}.
\end{abstract}

\section{Introduction}
Bayesian hierarchical models account for complicated relationships in data by introducing hierarchical structures \cite{gelman2006data}. 
Among hierarchical models, linear mixed effects models (LMMs) are widely used in various scientific disciplines, including ecology \cite{harrison2018brief}, medicine \cite{brown2014applied}, psychology \cite{meteyard2020best}, neuroscience \cite{yu2022beyond} and cognitive science \cite{nicenboim2021introduction}. \blue{Solving LMMs involves inferring latent variables, such as fixed and random effects, based on the observed data. Fixed effects are shared by all observations, while random effects vary across different groups within the data.} 
LMMs are often implemented using probabilistic programming languages (PPLs), which isolate inference from modeling: users write a program representing the model and the PPL automatically executes a suitable inference algorithm. 
Variants of Hamiltonian Monte Carlo (HMC) \cite{duane1987hybrid} are dominant in many PPLs today and are widely used for LMMs. 
For example, BRMS~\cite{burkner2017brms} is an influential R package that allows users to write regression-style formulas that are automatically translated to Stan programs~\cite{carpenter2017stan} representing an LMM, and then Stan's HMC implementation is called to generate posterior samples.

We develop techniques that allow users to easily transform their models to analytically marginalize random effect variables from LMMs to improve the efficiency of HMC.
Marginalization has several benefits. 
First, there are often pathologies in LMMs that hinder efficient HMC sampling. 
A notable one is the ``funnel'' shape created by correlation between variance parameters and parameters for fixed or random effects~\cite{neal2003slice}. 
Marginalization~\cite{lai2023automatically} and other program transformations~\cite{gorinova2020automatic} have been shown to be useful in addressing such pathologies.
Second, marginalization reduces the number $H$ of latent variables for HMC.
The complexity of HMC is about $\OO(H^{5/4})$~\cite{creutz1988global, neal2011mcmc}, so it is desirable to run HMC on a subset of variables if marginalization can be done efficiently.
Our methods enable marginalization of random effects in LMMs with a linear Gaussian structure, which includes models with normal and log-normal likelihoods as well as other likelihoods for continuous data based on transforming a normal distribution. \camread{Note that our methods are not limited to HMC, and could be applied to many inference algorithms.} \ds{Why "target a log density"? Couldn't they also apply to, e.g., Gibbs sampling? Maybe just say "many inference algorithms".}

There are several challenges to efficient marginalization.
The automatic marginalization algorithm of~\cite{lai2023automatically} can be applied to LMMs but is limited to scalar random variables, so it requires users to construct the LMM as a graphical model with separate variables for each effect and observation.
Another alternative is to model the relationships between effects and observations with a design matrix and marginalize effects using properties of multivariate normal distributions.
We call this the ``vectorized approach'' since it can leverage vectorization to accelerate computations.
Unfortunately, vectorized marginalization leads to a dense covariance matrix over the observations and thus cubic time for evaluating the log-density within HMC, when the log-density of the original could be evaluated in linear time.
Our main technical contribution is to accelerate vectorized marginalization for LMMs using fast linear algebra: we show that marginalization for a single random effect can be achieved with linear time complexity and can significantly accelerate HMC compared to both the original model and non-vectorized marginalization.


We implement vectorized marginalization for LMMs in NumPyro \cite{bingham2019pyro,phan2019composable} via simple classes users can use to express their models.
We evaluate our approach on a variety of real LMMs from past scientific investigations, including nine models and datasets from cognitive sciences, and find that marginalization is always beneficial. 
Our findings suggest that practitioners should marginalize group-level effects whenever applicable in Bayesian hierarchical inference.

\section{Background}
\label{sec:background}

To motivate our problem, we present an example model. 
In~\cite{wahn2016pupil}, a set of experiments were run to examine the relationship between human pupil and attention load. 
A total of $N=2228$ measurements of pupil sizes from $M=20$ subjects were taken under different attention load levels. Specifically, in the $i$th measurement, the pupil size $y_i\in \mathbb{R}^{+}$ of subject $g_i\in\{1,2,...,k\}$ under attention load $c_i\in \{0,1,2,3,4,5\}$ was recorded. 
Pupil size can be assumed to have linear relationship $y_i \approx \theta_0 + \theta_1 c_i$ with respect to the attention load $c_i$, where both the slope $\theta_1$ and intercept $\theta_0$ split into fixed and random effects:
\begin{align}
	y_i=\alpha + u_{g_i,1} + c_i(\beta + u_{g_i,2}) + \epsilon,\ \epsilon\sim\mathcal{N}(0, \sigma^2),\notag
\end{align}
where $\alpha,\beta$ are variables for fixed effects and $u_{\cdot,\cdot}$ are variables for subject-specific random effects. 
Bayesian hierarchical modeling assigns priors to each unknown variable:
\begin{gather}
	\alpha\sim\mathcal{N}(1000, 500^2),\ \beta\sim\mathcal{N}(0, 100),\ \sigma\sim\mathcal{N}^{+}(0, 1000),\ \TTT\sim\mathcal{N}^{+}(\zero,\text{diag}(1000^2,1000^2)), \notag\\
	\LLL_u\sim \LKJ(2,1),\ \camread{[u_{j,1},u_{j,2}]} \sim\mathcal{N}(\zero, \TTT\LLL_u\LLL_u^T\TTT),\ j=1,2,...,k.\notag
\end{gather}
A half-normal distribution ($\mathcal N^+$) and an LKJ distribution (LKJCholesky) \cite{lewandowski2009generating} are used as a prior on the covariance matrix. Inference for the unknown parameters determining the relationship between pupil size and attention load can be performed by writing a probabilistic program and running HMC.
For example, in NumPyro, the regression model for all measurements may be implemented as below.
\begin{minted}
	[
	autogobble,bgcolor=white,
	frame=single,
	fontsize=\footnotesize,
	]
	{python}
	numpyro.sample('y',dist.Normal(alpha+u[g][:,0]+c*(beta+u[g][:,1]),sigma),obs=y)
\end{minted}
The code above uses advanced indexing and vectorization techniques in \texttt{numpy}, where \texttt{u,g,c,y} are all vectors or matrices. We further observe that, conditioned on $\alpha,\beta,\sigma,\TTT,\LLL_u$, the distribution of all $\uu_j$ and all $y_i$ form a multivariate normal distribution. Theoretically it is possible to analytically integrate $\uu$ 
out from the model to improve inference efficiency. 
But it is not straightforward for users to transform the probabilistic program to do so, and, as we will see, if done in the most obvious way, may not make the model more efficient for HMC.

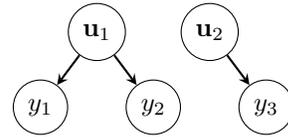
\begin{wrapfigure}{r}{0.28\textwidth}
	\vspace{-25pt}
	\begin{center}
		\begin{tikzpicture}
			\node[round,xshift=-0.75cm, yshift=-1cm](x1){$\uu_1$};
			\node[round,xshift=0.75cm, yshift=-1cm](x2){$\uu_2$};
			\node[round,xshift=-1.5cm, yshift=-2cm](y1){$y_1$};
			\node[round,xshift=0cm, yshift=-2cm](y2){$y_2$};
			\node[round,xshift=1.5cm, yshift=-2cm](y3){$y_3$};
			
			\draw [arrow, ->] (x1) -- (y1);
			\draw [arrow, ->] (x1) -- (y2);
			\draw [arrow, ->] (x2) -- (y3);
		\end{tikzpicture}
	\end{center}
	\caption{A tree-structured model conditioned on $\TT$.}
	\label{fig:tree}
	\vspace{-15pt}
\end{wrapfigure}
To be more clear about how marginalization can be implemented, we rearrange the model into a canonical form that focuses on the random effects. All observations are collected into the vector $\yy=[y_1,...,y_N]^T$ and random effects into the vector $\uu=[u_{1,1},u_{1,2},...,u_{k,1},u_{k,2}]^T$.
Then, we can write
\begin{align}
	\uu\sim\mathcal{N}(\mm,\SSU),\ \yy\sim\mathcal{N}(\A\uu+\bb,\SSY),\notag
\end{align}
where $\mm,\SSU,\A,\bb,\SSY$ are functions of $\alpha,\beta,\sigma,\TTT,\LLL_u,g_i,c_i$. 
Note that $y_i$ only depends on the entry $\uu_{g_i}$ of $\uu$.
The corresponding graphical model has a tree structure, as demonstrated in Figure \ref{fig:tree}.
This tree structure has several benefits: first, matrix multiplications like $\A\uu$ and $\A^T\yy$ can be done efficiently; second, we will see that it leads to a block-diagonal structure that facilitates efficient inversion in a key matrix that appears later. 

\ds{I made some light edits below. I tried to follow this convention: "class of random effects" (the set of parameters) but "classification of observations" (the partitioning of observations into groups)}
For more general LMMs with more than one class of random effects we generalize the canonical form as
\begin{gather}
\label{eq:model}
\TT\sim p(\TT),\quad \uu_i|\TT \sim  \mathcal{N}(\mm_i(\TT),\SSU_i(\TT)),\quad i=1,2,...,L\notag\\
\quad \yy|\TT, \uu_1,\uu_2,...\uu_L\sim\mathcal{N}\left(\sum_{i=1}^L\A_i(\TT)\uu_i+\bb(\TT),\SSY(\TT)\right),
\end{gather}
where $p(\TT)$ is the distribution for global variables (including fixed effects), $p(\uu_i|\TT)$ is the distribution for random effects and $p(\yy|\TT,\uu_1,...,\uu_L)$ is the distribution for observations. \camread{Notationally this generalization further adds an index to each random effect to specify its class.}
A user might specify the model directly in this canonical form, or in another syntax (e.g., the formula syntax of BRMS) that is compiled to this form. 
Each pair $(\uu_i, \A_i)$ specifies a \camread{class} of random effects for a particular classification of the observations (e.g., by subject, age, gender, etc.). \camread{Each classification contains multiple groups and different classifications are distinct from one another. 
Each observation belongs to one group for each classification.}
The vector $\uu_i=[\uu_{i,1}^T,\uu_{i,2}^T,...,\uu_{i,k_i}^T]^T$ contains random effects for the $i$th classification (e.g., subject, age, or gender), consisting of $k_i$ groups (e.g., one subject, age, or gender), with $\uu_{i,j}$ containing the random effects (e.g., slope and intercept) for the $j$th group.
We denote the number of observations as $\dim(\yy)=N$, and the number of random effects per group as $\dim(\uu_{i,j})=d$.
Any covariates---such as $c_i$ in the pupil size example---are considered constants and not represented in the notation.
In LMMs, the number $d$ is related to the number of covariates and is usually small.
The total number of random effects for $\uu_i$ is denoted as $\dim(\uu_i) = M_i =k_i d$.
The matrix $\A_i$ therefore has size $N\times M_i$, and encodes the group structure for $\uu_i$ by mapping random effects (together with covariates) to observations. \blue{Each row of $\A_i$ encodes the assignment of an observation to one group, so it has at most $d$ nonzero elements. Therefore, the complexity of computing $\A_i\uu_i$ is $\OO(Nd)$, as $\A$ has at most $Nd$ nonzero elements.}
Henceforth, we omit the dependence on $\TT$ for $\mm$, $\SSU$, $\A$, $\bb$, $\SSY$ for simplicity. 

\paragraph{Marginalizing $\uu_i$}
It is possible to analytically marginalize variables in this model: since the mean of $\yy$ is linear in each $\uu_i$ and all of these variables are normally distributed, the joint distribution of $(\yy, \uu_1, \ldots, \uu_L)$ is also multivariate normal.
We will focus for most of the paper on marginalizing the random effects $\uu_i$ for a single $i$ in order to leverage the tree structure mentioned earlier, but return in Section~\ref{sec:opt2} to the idea of marginalizing many effects.
Locally, $\uu_i$ and $\yy$ form the conditional distribution $p(\uu_{i},\yy|\TT,\uu_{-i})=p(\uu_i|\TT)p(\yy|\TT,\uu_{-i},\uu_i)$. 
Marginalized MCMC rewrites this conditional distribution as $p(\uu_{i},\yy|\TT,\uu_{-i})=p(\yy|\TT,\uu_{-i})p(\uu_{i}|\TT,\yy,\uu_{-i})$, which reverses the dependence between $\uu_i$ and $\yy$~\cite{lai2023automatically}.
During sampling, $\uu_i$ is marginalized from the HMC procedure by using $p(\yy|\TT,\uu_{-i})$ as the likelihood function and $p(\TT,\uu_{-i})$ as the distribution of latent variables. 
After HMC sampling, $\uu_i$ is recovered through ancestral sampling from $p(\uu_{i}|\TT,\yy,\uu_{-i})$ given posterior samples of $(\TT,\uu_{-i})$. 
The reversal requires analytical forms of $p(\yy|\TT,\uu_{-i})$ and $p(\uu_{i}|\TT,\yy,\uu_{-i})$, which can be obtained via standard marginalization and conditioning operations on multivariate normal distributions~\cite[e.g.,][]{bishopPRML}
\begin{align}
\label{eq:marginalized}
\yy|\TT,\uu_{-i}&\sim\mathcal{N}\left(\sum_{j\neq i}\A_j\uu_j+\A_i\mm_i+\bb,\A_i\SSU_i\A_i^T+\SSY\right),\notag\\
\uu_i|\TT,\yy,\uu_{-i}&\sim\mathcal{N}\left(\mm_i+\M\left(\yy-\sum_{j\neq i}\A_j\uu_j-\A_i\mm_i-\bb\right),(\I-\M\A_i)\SSU_i\right),
\end{align}
where $\M=\SSU_i\A_i^T(\A_i\SSU_i\A_i^T+\SSY)^{-1}$. 
Marginalization introduces the benefit of sampling in a lower dimensional space, but the cost depends on the complexity of evaluating the log-density functions of these two distributions in order to run HMC. 

\subsection{Challenges of multivariate marginalization}
In practice, the original model usually has structure that makes evaluating its density very efficient, which is lost by naive marginalization.
For example, the observations in $\yy$ are usually conditionally independent, making $\SSY$ diagonal; also, $\SSU_i$ is usually block diagonal with blocks of size $d\times d$. 
So evaluating the density $p(\uu_i,\yy|\TT,\uu_{-i})=p(\uu_i|\TT)p(\yy|\TT,\uu_{1:L})$ requires $\OO(k_id^3+NLd)=\OO(M_id^2+NLd)$ time with the main operations being (1) inverting and computing the determinant of $\SSU$ and $\SSY$; (2) computing the mean parameter of $\yy$.
When $\SSU_i$ is diagonal, the complexity goes down to $\OO(M_id+NLd)$. However, it is more expensive to evaluate the density of the reversed model in Equation (\ref{eq:marginalized}). 
Computing $p(\yy|\TT,\uu_{-i})$ and $p(\uu_i|\TT,\yy,\uu_{-i})$ requires the inverting and computing the determinant of the $N\times N$ matrix $\A_i\SSU_i\A_i^T+\SSY$, which we denote by $\EE$ for simplicity. For the log likelihood, we need to compute $\log p(\yy|\TT,\uu_{-i})=-\frac{1}{2}\det\left(\EE\right)-\frac{1}{2}\zz^T\EE^{-1}\zz + C$,
where $\zz=\yy-\sum_{j\neq i}\A_j\uu_j-\A_i\mm_i-\bb$.
$\EE$ is not diagonal and without using additional structure will trigger $\OO(N^3)$ operations within each step of the leapfrog integrator within HMC. 
For the recovery distribution $p(\uu_i|\TT,\yy,\uu_{-i})$, $\EE$ will be inverted when calculating $\M$. Also, a Cholesky decomposition for the covariance $(\I-\M\A_i)\SSU_i$ should be computed for sampling, which takes $\OO(M_i^3)$ time. 
These cubic time operations are prohibitively expensive for large datasets. \blue{We summarize the complexities of different approaches in Table \ref{tab:complexity}. In Section \ref{sec:marginalization}, we discuss how to marginalize one group of random effects with lemmas from linear algebra. In Section \ref{sec:opt2}, we discuss how to marginalize all random effects with additional assumptions.}

\section{Marginalization with fast linear algebra}
\label{sec:marginalization}
\begin{table}[t]
\centering
\caption{Time complexities of different HMC approaches for the submodel involved in marginalization. Initialization is done once before the HMC loop. The log density is computed within each step of the leapfrog integrator. Recovery is performed for each sample from HMC. $N$ is the number of observations, $M$ is the dimension for one class of random effects, $D$ is the dimension for all classes of random effects, $L$ is the number of classes, $d$ is the dimension for an effect of a group in a class.
}
\tiny
\begin{tabular}{ccccc}
	\toprule
	Submodel&Approach&Initialization&Log density&Recovery\\
	\midrule
	\multirow{3}{*}{$p(\uu_i,\yy|\TT,\uu_{-i})$}&No marginalization&-&$\OO(Md^2+NLd)$&-\\
	&Naive marginalization&-&$\OO(M^3+N^3)$&$\OO(M^3+N^3)$\\
	&Marginalize with lemmas&-&$\OO(Md^2+NLd+Nd^2)$&$\OO(Md^2+NLd+Nd^2)$\\
	\midrule
	\multirow{3}{*}{$p(\vv,\yy|\TT)$}&No marginalization&-&$\OO(Dd^2+NLd)$&-\\
	&Naive marginalization&-&$\OO(D^3+N^3)$&$\OO(D^3+N^3)$\\
	&Marginalize with assumptions&$\OO(D^3+NL^2d^2)$&$\OO(D^2+NLd)$&$\OO(D^2+NLd)$\\
	\bottomrule
\end{tabular}
\vspace{-10pt}
\label{tab:complexity}
\end{table}
We now show how to speed up calculations with the marginalized model using fast linear algebra methods. 
In particular, we use the matrix inversion lemma and matrix determinant lemma together with special structure in the relevant matrices.
In this section, we sometimes omit the subscript $i$ such as for $\A_i$ and $\SSU_i$ for simplicity. 
The steps in log density evaluation and recovery are summarized in Algorithm \ref{alg:marginalized_likelihiood}, and in Algorithm \ref{alg:recovery} in the appendix, \blue{with comments about their implementation and cost}.
\blue{We mainly use sparsity and tree-structure in $\A$ to make operations faster. As an overview, computing $\zz$ takes $\OO(NLd)$ time for $L$ sparse matrix multiplications of time $\OO(Nd)$ each. 
Also, evaluating $\A\sss$ and $\A^T\ttt$ both take $\OO(Nd)$ for any $\sss\in\mathbb{R}^{M}$ and any $\ttt\in\mathbb{R}^{N}$. 
With tree-structure, we will see that $\A^T\SSYI\A$ is block-diagonal and can be computed efficiently.
}

\subsection{Matrix inversion and determinant lemmas in marginalization}
The two main bottlenecks when evaluating $\log p(\yy|\TT,\uu_{-i})$ are computing $\det(\EE)$ and $\zz^T\EE^{-1}\zz$. With the matrix determinant lemma \cite{harville1998matrix}, we have that
\begin{align}
\label{eq:matrix_determinant_lemma}	\det(\EE)=\det(\A\SSU\A^T+\SSY)=\det(\SSUI+\A^T\SSYI\A)\det(\SSU)\det(\SSY).
\end{align}
By the matrix inversion lemma or the Woodbury formula \cite{petersen2008matrix} we have that
\begin{align}
\EE^{-1}=(\A\SSU \A^T+\SSY)^{-1}=\SSYI-\SSYI\A(\SSUI+\A^T\SSYI\A)^{-1}\A^T\SSYI.\notag
\end{align}
Therefore,
\begin{align}
\label{eq:matrix_inversion_lemma}
\zz^T\EE^{-1}\zz=\zz^T\SSYI\zz-\zz^T\SSYI\A(\SSUI+\A^T\SSYI\A)^{-1}\A^T\SSYI\zz.
\end{align}
By using the facts that $\SSU$ is block-diagonal, $\SSY$ is diagonal, and $\A$ has $Nd$ nonzero elements, the quantities $\det(\SSU)$, $\det(\SSY)$, $\zz^T\SSYI\zz$, and $\A^T\SSYI\zz$ can each be calculated in $\OO(Md^2+Nd)$ time. 
\ds{Unsupported}
Equations (\ref{eq:matrix_determinant_lemma}) and (\ref{eq:matrix_inversion_lemma}) contain the expressions $\FF^{-1}$ or $\det(\FF)$ for the $M \times M$ matrix $\FF:=\SSUI+\A^T\SSYI\A$, which both require $\OO(M^3)$ time when done naively. 
The following theorem shows that these quantities can be computed in $\OO((M+N)d^2)$ for LMMs.
\begin{theorem}
\label{theorem1}
If $\SSY$ is diagonal, $\SSU$ is block-diagonal with blocks of size $d\times d$, then $\FF=\SSUI+\A^T\SSYI\A$ is also block-diagonal with $d\times d$ blocks and computing $\A^T\SSYI\A$ takes $\OO(Nd^2)$. 
\end{theorem}
\begin{proof}
The proof uses the tree-structure in $\A$. For details, see Appendix \ref{sec:proof1}.
\end{proof}
Therefore, it is $\OO((M+N)d^2)$ to compute $\det(\FF)$ and $\FF^{-1}$. 
Combined with other parts in the formulas, the overall complexity is $\OO(Md^2+NLd+Nd^2)$. 
In LMMs, $d$ is usually small, so the complexity with marginalization can be viewed as the same as the complexity without marginalization. 
\subsection{Speeding up the recovery step}
\label{sec:rec}
Different from evaluating $\log p(\yy|\TT,\uu_{-i})$, ancestral sampling from $p(\uu_i|\TT,\yy,\uu_{-i})$ is only performed once for each posterior sample. 
When sampling from $p(\uu_i|\TT,\yy,\uu_{-i})$, computing $\M$ directly is also costly.
With the matrix inversion lemma, we have
\begin{align}
\M&=\SSU\A^T(\A\SSU\A^T+\SSY)^{-1}\notag\\
&=\SSU\A^T\SSYI-\SSU\A^T\SSYI\A(\SSUI+\A^T\SSYI\A)^{-1}\A^T\SSYI.  \label{eq:M}
\end{align}
With this expression, the mean variable $\mm+\M\zz$, then is evaluated in $\OO((M+N)d^2)$, by computing $\SSUI+\A^T\SSYI\A$ in the same way as Line 2 of Algorithm \ref{alg:marginalized_likelihiood}.
For the covariance variable $(\I-\M\A)\SSU$, we have from the reversed application of the matrix inversion lemma that 
\begin{align}
(\I-\M\A)\SSU&=\SSU-\SSU\A^T(\A\SSU\A^T+\SSY)^{-1}\A\SSU\notag\\
&=(\SSUI+\A^T\SSYI\A)^{-1}.\notag
\end{align}
Note that $\FF=\SSUI+\A^T\SSYI\A$ is all block diagonal. For a block diagonal matrix with $k$ blocks of size $d\times d$, the time complexity for a Cholesky decomposition is $\OO(kd^3)=\OO(Md^2)$. Combined with the complexity of computing $\zz$, the recovery step takes $\OO(Md^2+NLd+Nd^2)$ time.



\begin{algorithm}[t]
\caption{Evaluating $\log p(\yy|\TT,\uu_{-i})$. Each $\A_i$ is an $N\times M_i$ sparse matrix with $Nd$ elements and tree structure. $\SSY$ is $N\times N$ diagonal. $\SSU$ is $M\times M$ block-diagonal with block size $d$.
}\label{alg:marginalized_likelihiood}
\begin{algorithmic}[1]
	\State $\zz=\yy-\sum_{j\neq i}\A_j\uu_j-\A_i\mm_i-\bb$\Comment{Sparse matrix multiplication in $\OO(NLd)$ time}
	\State $\FF=\SSUI+\A^T\SSYI\A$\Comment{Block diagonal computation in $\OO((M+N)d^2)$ time}
	\State $\xx=\A^T\SSYI\zz$\Comment{Sparse matrix multiplication in $\OO(Nd)$ time}
	\State $a=\log\det(\FF)+\log\det(\SSU)+\log\det(\SSY)$\Comment{Determinants in $\OO(Md^2)$ time}
	\State $b=\zz^T\SSYI\zz-\xx^T\FF^{-1}\xx$ \Comment{Quadratic form in $\OO(N+Md)$ time}
	\State \Return $-\frac{1}{2}(a+b)+C$
\end{algorithmic}
\end{algorithm}

\section{Marginalizing multiple effects with additional assumptions}
\label{sec:prop}
\label{sec:opt2}
We have shown that it is efficient to marginalize one class of random effects. With additional practical assumptions, it is possible to marginalize all classes of random effects for efficient HMC inference. Instead of separating different classes of random effects, LMMs can also be written as 
$ \vv\sim\mathcal{N}(\mm,\SSV),\quad \yy\sim\mathcal{N}(\B\vv+\bb,\SSY)$,
where $\B=[\A_1,...,\A_L]$ and $\vv=[\uu_1^T,...,\uu_L^T]^T$. We define that $D=\sum_{i=1}^LM_i$. The matrix inversion and determinant lemmas can still be applied to marginalize $\vv$ out, but the combined matrix $\B$ does not have the special structure of $\A_i$ we exploited in Section~\ref{sec:marginalization}.
More specifically, the computation of $\det(\FF)$ and the evaluation of $\FF^{-1}$ for $\FF=\SSVI+\B^T\SSYI\B$ both become non-trivial. We introduce additional assumptions to show that they can be solved faster in some special cases. For the general case, see the discussion section.
The assumption we make is that $\SSV=\ssV\I$ and $\SSY=\ssY\I$, where $\ssV,\ssY$ are scalars that either belong to $\TT$ or are fixed non-random parameters. This means that all effects share the same variance and all observations share the same noise scale. These assumptions are not as restrictive as it may appear. 
If the underlying distribution is $\uu_i \sim\mathcal{N}(\mm_i,\sigma_i^2\I)$ where $\sigma_i$ is a fixed parameter, it is possible to reparameterize this distribution as $\uu_i'\sim\mathcal{N}(\zero,\I), \,\A_i'=\sigma_i\A_i, \,\bb'=\bb+\B\mm_i$, and use $\uu_i',\A_i',\bb'$ in place of $\uu_i,\A_i,\bb$. Then $\SSV$ becomes a scaled identity matrix. Also, in many models, the noise scale for different observations is the same, making $\SSY$ a scaled identity matrix as well.


In practice, if the assumptions are satisfied, marginalization can be done in $\OO(D^2+Nd)$ time with $\OO(D^3+NL^2d^2)$ preprocessing. Details are provided in Appendix~\ref{sec:detail_scaled}.

\section{Related Work}

While many works aim to improve HMC directly~\cite{ver2021hamiltonian, grumitt2022deterministic,robnik2023microcanonical,WangW23b}, a number of other works focus on model transformation. 
Non-centered parameterization \cite{papaspiliopoulos2007general} is a widely used trick among MCMC users to alleviate slow sampling in difficult posterior distributions.
However, there is no general way to know whether a non-centered parameterization will be beneficial~\cite{yao2018yes}. 
Variationally inferred parameterization \cite{gorinova2020automatic} proposes to learn a model parameterization from a specified family that will lead to effective sampling.
In \citet{parno2018transport} and \citet{hoffman2019neutra}, preconditioners for HMC are learned to transform the model to be approximately isotropic Gaussians.
Marginalization differs from reparameterization in that it reduces the problem dimension as well as potentially alleviating difficult characteristics such as funnels, so it has two mechanisms to improve MCMC efficiency.
The Laplace approximation (LA) is one way to approximately marginalize variables in MCMC~\cite{rue2009approximate, margossian2020hamiltonian, silverman2022bayesian}, but it may be difficult to quantify the error or recover the marginalized variables.

Marginalization, or Rao-Blackwellization, has been an important topic in Bayesian inference and probabilistic programming. In Gibbs sampling, marginalization is usually called collapsing~\cite{liu1994collapsed}. Collapsed Gibbs sampling has been developed for latent Dirichlet allocation \cite{porteous2008fast} and LMMs \cite{papaspiliopoulos2020scalable}. We explore marginalization in the context of HMC, which induces different considerations. \blue{Methods with HMC do not have to make the conditional distributions of the marginalized model tractable.} 
Marginalization is also related to symbolic inference in probabilistic programming. 
Hakaru \cite{narayanan2016probabilistic} and PSI \cite{gehr2016psi,gehr2020lambdapsi} are systems for performing exact Bayesian inference by symbolically marginalizing all latent variables. To marginalize discrete variables, \citet{gorinova2021conditional} propose an information flow type system. Another line of related work is delayed sampling \cite{murray2018delayed,atkinson2022semi}, which automates marginalization of variables within Rao-Blackwellized particle filters \cite{murphy2001rao}. \citet{lai2023automatically} developed an automatic system for marginalizing variables in HMC, but is limited to scalar variables so cannot leverage vectorization and forces users to write models with univariate distributions.

Linear algebra tricks have been widely utilized in various machine learning algorithms, such as ridge regression \cite{van2015lecture}, Gaussian processes \cite{seeger2004gaussian} and Kalman filters \cite{sarkka2023bayesian}. Recently, frameworks \cite{seeger2017auto,gardner2018gpytorch,potapczynski2024cola} have been proposed to ease the implementation of fast linear algebras in machine learning algorithms. Marginalization in Bayesian models may be an interesting application of those frameworks.

Fast and scalable inference for LMMs has been studied in the context of maximum likelihood estimation \cite{gao2020estimation}, variational EM \cite{ghandwani2023scalable}, Gibbs sampling \cite{papaspiliopoulos2023scalable} and numerical integration \cite{greengard2023fast}. We are the first to consider speeding up the inference of LMMs with HMC. There is also a recent trend in integrating random effects into deep neural networks for correlated data \cite{simchoni2023integrating} or personalization \cite{simchoni2021using,shi2022generalized, wortwein2023neural} with parameters estimated by maximum likelihood. 
\section{Experiments}
We conduct experiments on LMMs from various disciplines using the \camread{default} no-U-turn sampler (NUTS) \cite{hoffman2014no} from NumPyro \cite{bingham2019pyro,phan2019composable}, \camread{which has an adaptive step size with dual averaging, adaptive and diagonal mass matrix, target acceptance probability of 0.8, and maximum tree depth of 10. For the ETH instructor evaluation model, we set the maximum tree depth to 12 to overcome difficulties performing inference without marginalization in preliminary experiments. For all models, we use weakly informative priors unless specified. In general, our conclusion is insensitive to the choice of hyperparameters and priors.}
\ds{What do you mean by "orthogonal"? Maybe say "conclusions are insensitive to choices of".}
For all experiments, we collect 10,000 warm up samples for tuning, and 100,000 samples for evaluation, and evaluate performance via effective sample size
(ESS) and running time.  
\subsection{Marginalization in cross-effects models}
\label{sec:inst_eval}
Cross-effects models are a type of LMM that have more than one class of random effects (i.e. $L>1$). Usually each observation 
belongs to one subject group (e.g. individuals, animals) and one item group (e.g. questions, objects). 
The correlation among latent variables can create severely challenging geometry that slows down the sampling of HMC. With our idea, it is possible to marginalize one or more group of effects from the model, reducing the dimension of latent space for faster sampling and better geometry.

\textbf{ETH instructor evaluations} An example cross-effects model describes university lecture evaluations by students at ETH \cite{lme4}. The dataset records $N=73421$ ratings, where each rating $y_n$ comes from student $s_n$ for professor $p_n$ teaching a course from department $d_n$, with $t_n$ indicating whether the professor is teaching outside their own department. There are a total of $M_1=2972$ students, $M_2=1128$ professors and $M_3=14$ departments. We use a version of the model from the document of Tensorflow probability~\cite{dillon2017tensorflow}.
The model is
\begin{gather}
\text{Likelihood}: y_n\sim\mathcal{N}(u_{1,s_n}+u_{2,p_n}+u_{3,d_n}+\alpha+\beta t_n,\sigma^2),\notag\\
\text{Prior}: u_{1,i}\sim\mathcal{N}(0,1),\ u_{2,j}\sim\mathcal{N}(0,1),\ u_{3,k}\sim\mathcal{N}(0,1),\ \alpha\sim\mathcal{N}(0,5),\ \beta\sim\mathcal{N}(0,1),\ 
\sigma\sim\mathcal{N}^{+}(0,1),\notag
\end{gather}
where $1\le i\le M_1$, $1\le j\le M_2$ and $1\le k\le M_3$. Given the dataset, we wish to learn about the latent variables $\uu_1$, $\uu_2$, $\uu_3$, $\alpha$, $\beta$ and $\sigma$.
HMC is the most direct way to sample those variables, but the dimension and complicated relations make it inefficient. Marginalization can be applied to one of the effects, $\uu_1$, $\uu_2$ or $\uu_3$. 
We report the running time of sampling from the model with and without marginalization in Table \ref{tab:inst_time}. 
We found that marginalizing any group of random effects improves the sampling speed of HMC.
However, the improvements are not necessarily predicted by the dimension of marginalized variable: HMC is faster when marginalizing $\uu_3$ than when marginalizing $\uu_1$ even though $\uu_1$ has 200-times higher dimension than $\uu_3$.
In Figure \ref{fig:inst_ess}, the ESS for each variable is reported.
Without marginalization, sampling $\uu_2$ and $\uu_3$ are both difficult compared to sampling $\uu_1$, and HMC becomes more efficient when marginalizing either of these variables, so we conjecture that $\uu_2$ and $\uu_3$ are responsible for the difficulty for sampling in the original model.
In this model, all random effects are independent and have the same variance, so $\SSU$ is a scaled identity matrix and we can marginalize all random effects efficiently.
This approach is observed to be the most efficient in our experiments, despite having quadratic complexity in $D$. Overall, marginalization never hurts ESS, and runs faster. We expect that any marginalization strategy works better than HMC in the original model, a finding which will be consistent across experiments. \camread{Additional results of this experiment, including trace plots and $\hat{R}$ diagnosis, are included in Figure \ref{fig:inst_trace} and Table \ref{tab:inst_rhat} in the Appendix.}
\begin{table}[t]
\centering
\caption{Running time in seconds for HMC, with or without marginalization. Mean and standard deviation over 5 independent runs are reported. Experiments are run on NVIDIA A40.}
\small
\begin{tabular}{cccccc}
	\toprule
	Method&No marginalization&Marginalize $\uu_1$&Marginalize $\uu_2$&Marginalize $\uu_3$&Marginalize $\uu$\\
	\midrule
	Time (s)&13417 (98)&5004 (1468)&2607 (3)&3071 (4)&\textbf{631} (12)\\
	\bottomrule
\end{tabular}
\label{tab:inst_time}
\end{table}
\begin{figure}[t]
\centering
\includegraphics[width=\textwidth]{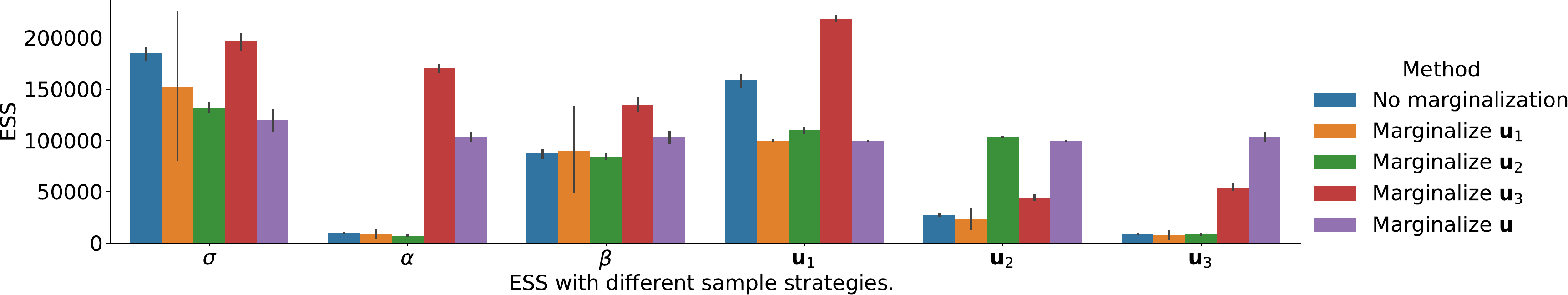}
\vspace{-10pt}
\caption{Average ESS for each variable on the instruction evaluation model with different HMC strategies. Numbers above the sample size 100,000 indicate effective sampling. }
\vspace{-10pt}
\label{fig:inst_ess}
\end{figure}

\subsection{Marginalization vs reparameterization}
To tackle bad geometry in statistical models, another model transformation is non-centered parameterization, or reparameterization \cite{papaspiliopoulos2007general}. 
Reparameterization converts the distribution of $z\sim\mathcal{N}(\mu,\sigma^2)$ into $\epsilon\sim\mathcal{N}(0,1)$ and $z=\epsilon\sigma+\mu$. Reparameterization is especially useful for funnel shapes in hierarchical models. We note that when applicable, marginalization is able to solve a broader class of problems. We compare marginalization and reparameteriation on the grouse ticks model.

\textbf{Grouse ticks} The dataset \cite{lme4} contains observations $\yy$ of the the number of ticks on the heads of red grouse chicks in the field. 
Each observation $y_k$ comes from brood $b_k$ in location $l_k$ during year $e_k$ at altitude $a_k$, where year and altitude give fixed effects, and there are random effects $\uu_1$ and $\uu_2$ corresponding to brood and location. 
There are $N=403$ observations, $M_1=118$ broods and $M_2=63$ locations. We define the hierarchical model as follows:
\begin{gather}
\text{Likelihood}: y_k\sim\mathcal{N}(u_{1,b_k}+u_{2,l_k}+\beta_ee_k+\beta_aa_k,\sigma_t^2)\notag\\
\text{Prior}:\mu_{1} \sim\mathcal{N}(0,1),\ \sigma_{1}\sim \HC(5),\ \mu_{2} \sim\mathcal{N}(0,1),\ \sigma_{2}\sim \HC(5), \notag\\
\beta_e\sim\mathcal{N}(0,1),\ \beta_a\sim\mathcal{N}(0,1),\ u_{1, i}\sim\mathcal{N}(\mu_{1}, \sigma_{1}^2),\ u_{2, j}\sim\mathcal{N}(\mu_{2}, \sigma_{2}^2),
\sigma_t\sim\HC(5),\notag
\end{gather}
where $i=1,...,M_1$, $j=1,...,M_2$, $k=1,...,N$ and each $y_k$ is observed. 
The correlation between $\sigma$ and $\uu$ creates the funnel shape that makes vanilla HMC inefficient. Nevertheless, it is possible to apply either marginalization or reparameterization to each random effect. In Figure \ref{fig:grouse}, we plot the distributions of samples for variable pairs $(\sigma_1, u_{1,1})$ and $(\sigma_2, u_{2,1})$ with different combinations of marginalization and reparameterization. 
There is a difficult correlation between $\sigma_2$ and $\uu_2$. 
After applying marginalization or reparameterization to $\uu_2$, HMC manages to explore the funnel region (at low values of $\sigma_1$).
However, we find that only samplers that marginalize $\uu_2$ report zero divergent transitions after warm-up. \camread{Such behavior is consistent with different random seeds. See Table \ref{tab:grouse_divergence} in the Appendix.}
Also, \camread{the distribution of divergent samples is related to specific parameters when reparameterizing $\uu_2$, implying that reparameterization introduces pathologies that create challenges for HMC inference.} 
In addition, we find that reparameterization does not improve the running time of HMC, while marginalizing $\uu_2$ speeds up sampling by about 20\%. 

\begin{figure}[t]
\centering
\includegraphics[width=\textwidth]{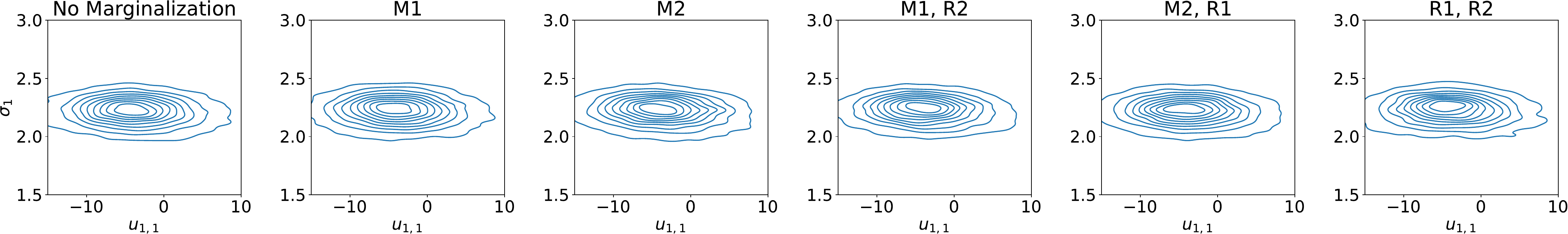}
\includegraphics[width=\textwidth]{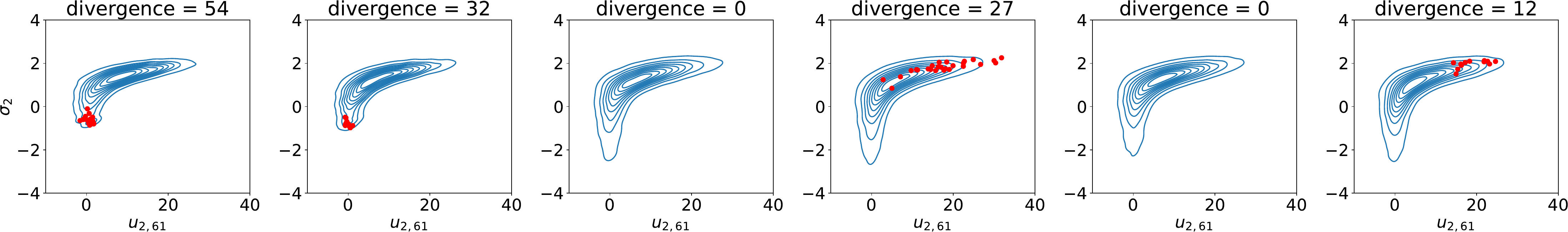}
\vspace{-10pt}
\caption{Distribution of 10,000 samples for variable pairs $(\sigma_1, u_{1,1})$ and $(\sigma_2, u_{2,61})$ on the grouseticks model with different methods. We use M1 to represent marginalizing $\uu_1$, M2 to represent marginalizing $\uu_2$, R1 to represent reparameterizing $\uu_1$, R2 to represent reparameterizing $\uu_2$. \camread{The number of divergences for each case are reported, with locations shown as red dots. We choose $u_{2,61}$ to demonstrate the distribution of divergences when reparameterizing $\uu_2$.} 
}
\vspace{-5pt}
\label{fig:grouse}
\end{figure}
\subsection{Benefits from vectorization}
In theory, marginalization with LMMs can be done by constructing a graphical model for scalar random variables and performing automatic marginalization as in \cite{lai2023automatically}. But it is more efficient to marginalize in a vectorized way. We demonstrate the benefits from vectorization in Table \ref{tab:time_vec}. Both marginalization strategies are performed on two hierarchical linear regression models, the electric company model \cite{gelman2006data} and the pulmonary fibrosis model \cite{osic-pulmonary-fibrosis-progression}. We find that vectorized marginalization is much more efficient for sampling from the two models.
\begin{table}[t]
\centering
\caption{Compilation time $T_c$ and running time $T_r$ in seconds for marginalized MCMC \cite{lai2023automatically}, with or without vectorization. Mean and std across 5 independ runs are reported.}
\small
\begin{tabular}{ccccc}
	\toprule
	Model&$T_c$ of \cite{lai2023automatically}&$T_r$ of \cite{lai2023automatically}&$T_c$ of ours&$T_r$ of ours\\
	\midrule
	Electric company&552 (4)&1249 (95)&7 (0)&252 (23)\\
	Pulmonary fibrosis&727 (11)&2208 (80)&10 (1)&178 (3)\\
	\bottomrule
\end{tabular}
\vspace{-10pt}
\label{tab:time_vec}
\end{table}
\subsection{Applications in cognitive sciences}
Hierarchical Bayesian inference with LMMs 
has wide applications in cognitive science~\cite{nicenboim2021introduction}. We highlight the effectiveness of marginalization with 9 datasets from cognitive science (Table \ref{tab:spec}). They cover various settings, with one or two random effects, normal or log-normal likelihoods, on CPU or GPU. Experiments that are slow on CPU are performed on GPU. 
Each dataset corresponds to an LMM where both the intercept and the coefficient include random effects. 
Details of all the models can be found in Appendix \ref{sec:models}. Results are summarized in Figure \ref{fig:cog_time}. Marginalization usually improves the sampling speed of HMC and consistently improves efficiency measured by ESS per iteration.

\begin{table}[t]
\tiny
\caption{Specifications of the datasets from cognitive sciences. Details of each model are provided in Appendix \ref{sec:models}. GPU models run on an NVIDIA RTX 2080ti GPU. CPU models run on one Intel Xeon Gold 6148 processor.}
\begin{tabular}{cccccccccc}
	\toprule
	&dillonE1\cite{dillon2013contrasting}&
	dutch\cite{frank2016cross}&
	eeg\cite{nieuwland2018large}&
	english\cite{vasishth2010short}&
	gg05\cite{grodner2005consequences}&
	mandarin\cite{wu2008processing}&
	mandarin2\cite{vasishth2013processing}&
	pupil\cite{wahn2016pupil}&
	stroop\cite{ebersole2016many}\\
	\midrule
	$N$&2855&372&26176&768&672&547&595&2228&3058\\
	$L$&2&2&2&2&2&2&2&1&1\\
	$M_1$&40&24&334&48&42&37&40&20&50\\
	$M_2$&48&16&80&16&16&15&15&-&-\\
	Likelihood&LogNormal&Normal&Normal&Normal&LogNormal&LogNormal&LogNormal&Normal&LogNormal\\
	Device&GPU&GPU&GPU&CPU&CPU&CPU&GPU&GPU&GPU\\
	\bottomrule
\end{tabular}
\label{tab:spec}
\end{table}
\begin{figure}[t]
\centering
\includegraphics[width=\textwidth]{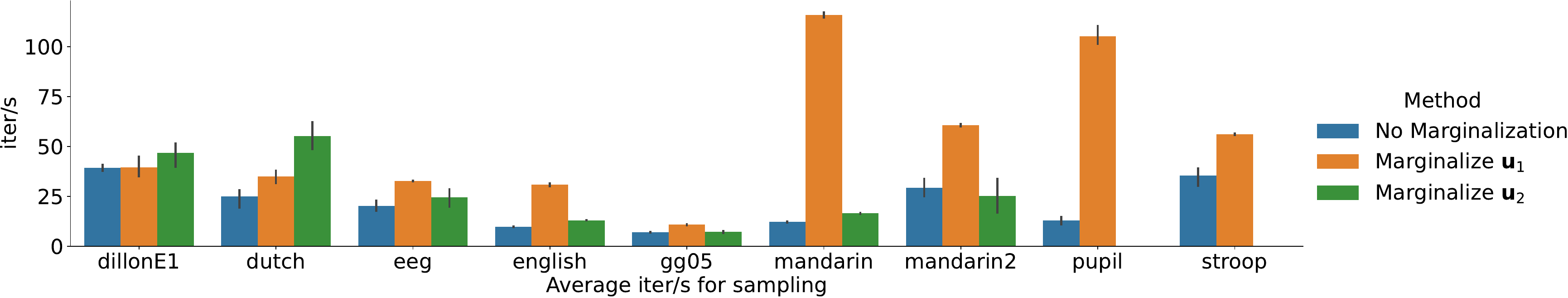}
\includegraphics[width=\textwidth]{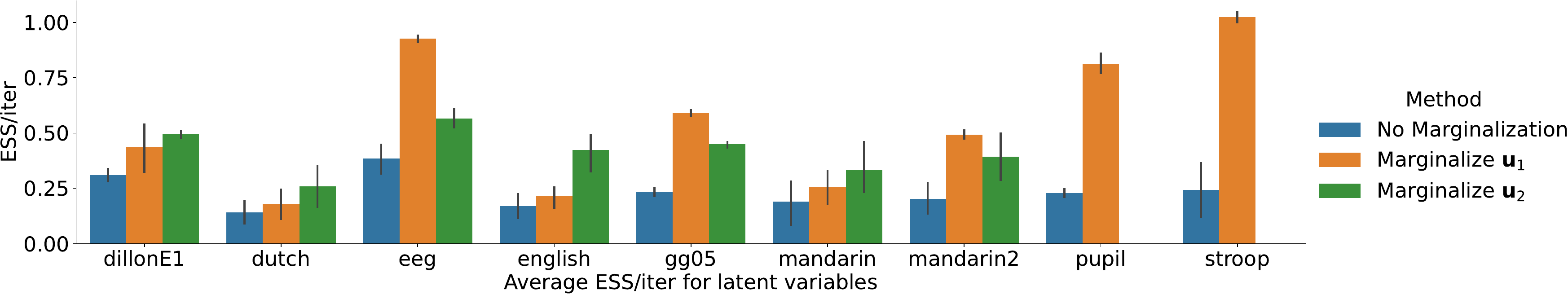}
\vspace{-10pt}
\caption{Experimental results for the 9 cognitive science datasets with and without marginalization. Each experiment is performed 5 times with different random seeds. Marginalization usually improves sampling speed measured by iterations per second (iter/s) and sample efficiency measured by ESS per iteration (ESS/iter).}
\vspace{-10pt}
\label{fig:cog_time}
\end{figure}
\section{Discussion}
There are several promising directions for future work.
\subsection{\camread{Marginalization vs Rao-Blackwellization}}
\camread{
Marginalization is related to Rao-Blackwellization.
This paper focuses on marginalization, which improves the speed of obtaining samples from the remaining variables by improving mixing times, reducing the cost per iteration, or both. 
Combining marginalization with Rao-Blackwellization is an interesting avenue for future work.
More formally, if one is interested in some expectation $\mathbb{E}_{(\TT,\uu)\sim p(\TT,\uu|\yy)}[f(\TT,\uu)]$ in an LMM, there is a Monte Carlo estimator
$$
E_1=\frac{1}{N} \sum_{i=1}^N f(\TT^i, \uu^i),
$$
where $(\TT^i, \uu^i)\sim p(\TT, \uu|\yy)$ and $N$ is the sample size. Marginalization is a trick to improve the efficiency of the posterior sampling, so that we can achieve the same estimation variance with smaller $N$ or less runtime\ds{Add "or less runtime"?}. 
At the same time, we also have access to a conditional distribution that is useful for Rao-Blackwellization. 
If the effects variable $\uu$ can be marginalized we have both an approximate posterior for $p(\TT|\yy)$ and an analytical conditional distribution $p(\uu|\TT,\yy)$. 
With Rao-Blackwellization we have that $\mathbb{E}_{(\TT,\uu)\sim p(\TT,\uu|\yy)}[f(\TT,\uu)]=\mathbb{E}_{\TT\sim p(\TT|\yy)}[\mathbb{E}_{\uu\sim p(\uu|\TT,\yy)}[f(\TT,\uu)]]$. In such case, another Monte Carlo estimator can be constructed:
$$
E_2=\frac{1}{N} \sum_{i=1}^N \mathbb{E}_{\uu\sim p(\uu|\TT,\yy)}\left[f(\TT^i, \uu)\right],
$$
where $\TT^i\sim p(\TT|\yy)$. 
For some functions, such as those that are polynomial in $\uu$, the inner expectation can be computed exactly using properties of Gaussians. 
In other cases, the inner expectation can be estimated cheaply via Monte Carlo using exact samples from $p(\uu | \TT_i, \yy)$.
}
\subsection{Marginalizing multiple effects in general models}
In Section \ref{sec:opt2}, we proposed to marginalize multiple classes of random effects by assuming a scaled identity covariance matrix. 
To marginalize multiple effects in general models, a possibility is to compute $\zz^T\EE^{-1}\zz$ and estimate $\det(\EE)$ and the corresponding gradients with conjugate gradient (CG) solvers \cite{domke2012generic, gardner2018gpytorch}. 
However, this approach uses stochastic estimators for the determinant and gradients, which introduce bias into the HMC dynamics.
These biases can be corrected through pseudo-marginalization \cite{andrieu2009pseudo}, but it is unclear how significantly the extra stochasticity will affect the sampling. 
Another possible way to marginalize multiple effects for LMMs is to introduced the balanced levels assumption~\cite{papaspiliopoulos2020scalable}. 
We leave these ideas for future exploration.

\subsection{Beyond normal \camread{likelihoods}}
In this work, we only consider normal or log-normal likelihoods, but our method can be easily generalized to other deterministic transformation of normal likelihood. 
This implies that marginalization can benefit regression with most continuous predictors given proper link functions. Another potential future direction is to marginalize classification models with probit regressions \cite{agresti2015foundations}. Marginalization will turn probit models into multivariate probit models as $\A\SSU\A^T+\SSY$ is a dense covariance matrix, which may require a simulation-based method \cite{chib1998analysis} or variational Bayes \cite{mandt2017sparse}. It will be interesting to see how ideas from multivariate probit regression could be fit into an HMC pipeline. \camread{In a broader context, marginalization is related to data augmentation techniques that "create" conjugacy for non-normal likelihoods or non-normal effects. Those techniques were developed for Gibbs sampling, e.g. \cite{fruhwirth2009improved,polson2013bayesian}, but may also be useful for HMC. }
\ds{Risky (and not necessary) to say "but have not been explored in the context of HMC". Suggest "and may also be useful for HMC".}

\subsection{Integration with probabilistic programming}
We have developed a tool to speed up the HMC inference for LMMs. In our implementation, the marginalized likelihood $p(\yy|\TT,\uu_{-i})$ is defined as a special type of parametric distribution available to the user, and the recovery distribution $p(\uu_i|\TT,\uu_{-i},\yy)$ is a function called after sampling. In our experiments, marginalization never hurt sampling efficiency measured by ESS/s, and usually helped.
Thus, it would be desirable to always marginalize one group of random effects when the model is an LMM. 
Future work could aim to automatically apply such transformations to user-specified LMMs.
\ds{I made an editing pass below}
\camread{There are two possible high-level approaches. 
The first is to perform marginalization starting with a model described using a high-level abstraction such as an R formula. Then, when compiling the high-level model description into a concrete model (e.g., a probabilistic program), we can marginalize one or more of the effects using our methods.
The second is to perform marginalization starting with a user-written probabilistic program representing an LMM. 
In this case, some compilation or program tracing technique will be needed to convert the user’s program to a model representation suitable for manipulation. For example, \citet{lai2023automatically} used program tracing to construct a graphical model representation that could be programmatically analyzed and transformed. To apply this methodology to LMMs, a special parser would also be needed to match the models to LMMs. 
}

\section*{Acknowledgement}
\camread{The authors thank Yuling Yao and the anonymous reviewers for comments that greatly improved the manuscript. This material is based upon work supported by the National Science Foundation under Grants \#1749854, \#2045900.}
\bibliographystyle{plainnat}
\bibliography{reference}


\newpage
\appendix
\setcounter{theorem}{0}
\section{Notation table}
We summarize the important symbols used in the paper.
\begin{table}[H]
\begin{tabular}{ll}
	\toprule
	Symbols&Description\\
	\midrule
	$N$&Number of observations, and dimension of $\yy$\\
	$M$, $M_i$&Dimension for all effects in one class of mixed effects\\
	$L$&Number of classes of mixed effects\\
	$D$&Dimension for all mixed effects\\
	$d$&Dimension for effects of a group in a class\\
	$k$, $k_i$&Number of groups in a class\\
	$\alpha$&Intercept for linear regression\\
	$\beta$&Slope for linear regression\\
	$\sigma$&Standard deviation\\
	$u$, $\uu$&Random effects\\
	$\vv$&Concatenated random effects\\
	$y$, $\yy$&Observations\\
	$c$, $t$&Covariates, or treatments\\
	$\TTT$&A prior variable sampled from half-normal distributions\\
	$\LLL$&A prior variable sampled from LKJ distributions\\
	$g$&Grouping variables\\
	$\TT$&Global variables, including priors and fixed effects\\
	$\mm$&Mean of random effects\\
	$\A$&Design matrix for random effects\\
	$\B$&Concatenated design matrices\\
	$\bb$&Intercept term in the canonical form for LMMs\\
	$\SSU$&Covariance matrix for a class of random effects\\
	$\SSY$&Covariance matrix for the observations\\
	$\SSV$&Covariance matrix for all random effects\\
	$\ssV$&Scale for $\SSV$ with the scaled identity assumption\\
	$\ssY$&Scale for $\SSY$ with the scaled identity assumption\\
	$\M$&A shared matrix in the reversed model\\
	$\zz$&Difference between observation and mean of the marginalized likelihood\\
	$\EE$&A dense $N\times N$ matrix that is difficult to directly compute\\
	$\FF$&The core matrix after applying the two linear algebra lemmas\\
	$\GG$&An intermediate matrix in the implementation\\
	$\xx$&An intermediate vector in the implementation\\
	$\rr$&A row of $\A$\\
	$\cc$&A column of $\A$\\
	$\C$&A block of $d$ columns of $\A$\\
	$\Q$&The eigenvector matrix for eigendecompsition of $\B^T\B$\\
	$\LL$& The eigenvalue matrix for eigendecomposition of $\B^T\B$\\
	\bottomrule
\end{tabular}
\label{tab:notation}
\end{table}
\newpage
\section{Proofs and details}
\subsection{Proof of Theorem 1}
\label{sec:proof1}
We first review the tree structure of the matrix $\A$. $\A$ is an $N\times M$ matrix where every \blue{block of $d$ columns corresponds to the effects for one group (e.g., an individual subject, age, school, or gender)}. For example, if $N=3$, $k=2$ and $d=2$, one possible graphical model is as below.
\begin{figure}[H]
\begin{center}
	\begin{tikzpicture}
		\node[round,xshift=-0.75cm, yshift=-1cm](x1){$\uu_1$};
		\node[round,xshift=0.75cm, yshift=-1cm](x2){$\uu_2$};
		\node[round,xshift=-1.5cm, yshift=-2cm](y1){$y_1$};
		\node[round,xshift=0cm, yshift=-2cm](y2){$y_2$};
		\node[round,xshift=1.5cm, yshift=-2cm](y3){$y_3$};
		
		\draw [arrow, ->] (x1) -- (y1);
		\draw [arrow, ->] (x1) -- (y2);
		\draw [arrow, ->] (x2) -- (y3);
	\end{tikzpicture}
\end{center}
\caption{A tree-structured model conditioned on $\TT$.}
\label{fig:tree-supple}
\end{figure}
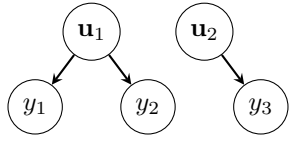
Each $\uu_j\in\mathbb{R}^2$. If the coefficients are all $1$s, then 
\begin{align}
\A=\left(\begin{matrix}
	1&1&0&0\\
	1&1&0&0\\
	0&0&1&1
\end{matrix}
\right).\notag
\end{align}
To generalize, if for $y_i$, the grouping variable is $g_i$, then in the $i$th row of $\A$, only \blue{$\A_{i,j:k}$ can be nonzero for $j=(g_i-1)d+1$ and $k=g_id$}. We consider \blue{three} representations of the matrix $\A$. By rows, 
\begin{align}
\A=\left(\begin{matrix}
	\rr_{1}\\
	\rr_{2}\\
	...\\
	\rr_{N}
\end{matrix}
\right),\notag
\end{align}
\ds{Minor: I think it would be more common to keep everything a column vector and use $\rr_1^T$ through $\rr_N^T$ for the rows. I think it would lead to more familiar expressions--- like $\rr_i (\rr'_i)^T$ for an outer product, instead of $\rr_i^T \rr_i'$, which looks like an inner product---later.}
\blue{by columns,
\begin{align}
	\A &=\left(\begin{matrix}
		\cc_1 &
		\cc_2&
		...&
		\cc_{kd}
	\end{matrix}
	\right)
	,\notag
\end{align}
and by blocks of $d$ columns,}
\begin{align}
\A
&=\left(\begin{matrix}
	\A_{:,1:d}&
	\A_{:,d+1:2d}&
	...&
	\A_{:,(k-1)d+1:kd}
\end{matrix}
\right)\notag\\
&=\left(\begin{matrix}
	\C_{1}&
	\C_{2}&
	...&
	\C_{k}
\end{matrix}
\right)\notag
\end{align}
where each $\C_i$ ($i=1,2,...,k$) is $N\times d$. Now we restate and prove Theorem \ref{theorem1}.
\begin{theorem}
If $\SSY$ is diagonal, $\SSU$ is block-diagonal with blocks of size $d\times d$, then $\FF=\SSUI+\A^T\SSYI\A$ is also block-diagonal with $d\times d$ blocks and computing $\A^T\SSYI\A$ takes $\OO(Nd^2)$ time.
\end{theorem}
\begin{proof}
The theorem has two parts: (a) the property of $\FF$, and (b) the computation of $\FF$. We address them with the three representations of $\A$. 

\textbf{(a) $\FF$ is block-diagonal.} Because $\SSU$ is block-diagonal, $\SSUI$ is also block-diagonal with the same sizes. Also, $\SSY$ is diagonal, so the block-diagonality of $\A^T\SSYI\A$ is the same as $\A^T\A$. We consider the column representation of $\A$, then
\begin{align}
	\A^T\A&=\left(\begin{matrix}
		\C_1^T\\
		\C_2^T\\
		...\\
		\C_k^T
	\end{matrix}
	\right)\left(\begin{matrix}
		\C_1&
		\C_2&
		...&
		\C_k
	\end{matrix}
	\right)\notag\\
	&=\left(\begin{matrix}
		\C_1^T\C_1&\C_1^T\C_2&...&\C_1^T\C_k\\
		\C_2^T\C_1&\C_2^T\C_2&...&\C_2^T\C_k\\
		...&\\
		\C_k^T\C_1&\C_k^T\C_2&...&\C_k^T\C_k
	\end{matrix}
	\right).\notag
\end{align} 
For $1\le i\le k$, $\C_i^T\C_i$ is $d\times d$. For $1\le i<j\le k$, 
\begin{align}
	\C_i^T\C_j&=\left(\begin{matrix}
		\cc_{(i-1)d+1}^T\\
		\cc_{(i-1)d+2}^T\\
		...\\
		\cc_{id}^T
	\end{matrix}
	\right)\left(\begin{matrix}
		\cc_{(j-1)d+1}&
		\cc_{(j-1)d+2}&
		...&
		\cc_{jd}
	\end{matrix}
	\right).\notag
\end{align}
The following lemma shows that $\C_i^T\C_j=\zero$.
\begin{lemma}
	For any $1\le i<j\le k$ and $1\le s,t\le d$, it holds that $\cc_{(i-1)d+s}^T\cc_{(j-1)d+t}=0$. 
\end{lemma}
\begin{proof}
	The lemma can be proved by contradiction. Suppose $\cc_{(i-1)d+s}^T\cc_{(j-1)d+t}\neq 0$. Then there exists an index $n$ such that $\cc_{(i-1)d+s}[n]\neq 0$ and $\cc_{(j-1)d+t}[n]\neq 0$. This means that in the $n$th row of $\A$, both $\A_{n,(i-1)d+s}$ and $\A_{n,(j-1)d+t}$ are non-zero. This contradicts with the tree-structure where only one group of $d$ elements can be non-zero in a row. 
\end{proof}
With the lemma, we have that $\A^T\A$ is block-diagonal, thus $\SSUI+\A^T\SSYI\A$ is also block-diagonal and each block is $d\times d$.

\textbf{(b) The computation of $\A^T\SSYI\A$ is $\OO(Nd^2)$.} Since $\SSY$ is diagonal, $\A'=\SSYI\A$ has the same \blue{pattern of zeros and nonzeros as $\A$}. We consider the row representations such that 
\begin{align}
	\A'=\left(\begin{matrix}
		\rr'_{1}\\
		\rr'_{2}\\
		...\\
		\rr'_{N}
	\end{matrix}
	\right).\notag
\end{align}
Then
\begin{align}
	\A^T\A'&=\left(\begin{matrix}
		\rr_{1}^T&
		\rr_{2}^T&
		...&
		\rr_{N}
	\end{matrix}
	\right)\left(\begin{matrix}
		\rr'_{1}\\
		\rr'_{2}\\
		...\\
		\rr'_{N}
	\end{matrix}
	\right)=\sum_{i=1}^N\rr_i^T\rr'_i.\notag
\end{align}
note that each of $\rr_i$ and $\rr'_i$ has $d$ non-zero elements. So computing $\A^T\A'$ is $\OO(Nd^2)$. 
\end{proof}

\subsection{Pseudocode for recovery after marginalizing one group of random effects}
\begin{algorithm}[H]
\caption{Sampling from $p(\uu_i|\TT,\yy,\uu_{-i})$}\label{alg:recovery}
\begin{algorithmic}
	
	\State $\zz=\yy-\sum_{j\neq i}\A_j\uu_j-\A_i\mm_i-\bb$\Comment{Sparse matrix multiplication in $\OO(NLd)$ time}
	\State $\GG=\A^T\SSYI\A$\Comment{Block diagonal computation in $\OO(Nd^2)$ time}
	\State $\FF=\SSUI+\GG$\Comment{Block diagonal computation in $\OO(Md^2)$ time}
	\State $\mm=\mm_i+\SSU(\I-\GG\FF^{-1})\A^T\SSYI\zz$\Comment{Sparse matrix multiplication in $\OO((M+N)d)$ time}
	\State $\LLL=\text{Cholesky}(\FF^{-1})$\Comment{Cholesky of block diagonal matrix in $\OO(Md^2)$ time}
	\State \Return $\uu_i\sim\text{Normal}(\mm,\LLL\LLL^T)$.
	
\end{algorithmic}
\end{algorithm}

\subsection{Details of scaled identity covariance matrices}
\label{sec:detail_scaled}
\ds{I only skimmed the math in this section. But it reads well, and seems to justify running time statements well.}
With the assumptions of scaled identity covariance matrices, all effects can be marginalized with a preprocessing of the eigendecomposition of $\B^T\B$. 

\textbf{(a) Preprocessing before HMC.} We compute
\begin{align}
\B^T\B=\Q\LL\Q^T.\notag
\end{align}

In LMMs, the computation of $\B^T\B$ is $\OO(NL^2d^2)$\footnote{Each $\A_i^T\A_j$ is $\OO(Nd^2)$, as a corollary of Theorem \ref{theorem1}.}, and the eigendecomposition of it is $\OO(D^3)$. So the overall complexity for preprocessing is $\OO(D^3+NL^2d^2)$. Compared with the HMC sampling loop that takes thousands of steps and visits the model hundreds of times each step, the cost of preprocessing is not expensive. In our attempt to marginalize all random effects for the instructor evaluation model in Section \ref{sec:inst_eval}, this step takes less than 10 seconds.

\textbf{(b) Marginalized likelihood during HMC.}
During HMC sampling, the log density $\log p(\yy|\TT)$ would be calculated, which is
\begin{align}
\log p(\yy|\TT)=-\frac{1}{2}\det\left(\EE\right)-\frac{1}{2}\zz^T\EE^{-1}\zz + C\notag
\end{align}
where $\zz=\yy-\B\mm-\bb$ and $\EE=\B\SSV\B^T+\SSY$. The computation of $\zz$ takes $\OO(NLd)$ time. With the two lemmas, we have
\begin{gather}
\det(\EE)=\det(\SSVI+\B^T\SSYI\B)\det(\SSV)\det(\SSY),\notag\\
\zz^T\EE^{-1}\zz=\zz^T\SSYI\zz-\zz^T\SSYI\B(\SSVI+\B^T\SSYI\B)^{-1}\B^T\SSYI\zz. \notag
\end{gather}

A shared matrix in the formulas is $\FF=\SSVI+\B^T\SSYI\B$. Then 
\begin{align}
\FF=\SSVI+\B^T\SSYI\B=\Q\left(\frac{1}{\ssV}\I+\frac{1}{\ssY}\LL\right)\Q^T.\notag
\end{align}
With the trick, evaluating $\det(\FF)$ reduced to $\OO(D)$ time as $\det(\Q)=1$. Also $\zz^T\EE^{-1}\zz$ becomes
\begin{align}
\zz^T\EE^{-1}\zz=\zz^T\SSYI\zz-\zz^T\SSYI\B\Q\left(\frac{1}{\ssV}\I+\frac{1}{\ssY}\LL\right)^{-1}\Q^T\B^T\SSYI\zz.\notag
\end{align}
Note that $\B^T\SSYI\zz\in\mathbb{R}^D$ can be computed in $\OO(NLd)$ time, but its multiplication with $\Q^T$ takes $\OO(D^2)$ time. Given that $\frac{1}{\ssV}\I+\frac{1}{\ssY}\LL$ is diagonal, the complexity of evaluating $\log p(\yy|\TT)$ once is then $\OO(D^2+NLd)$. 

\textbf{(c) Ancestral sampling after HMC.} 
In the recovery step, we perform ancestral sampling from $p(\vv|\TT,\yy)$. To efficiently generate samples, we give the following theorem.
\begin{theorem}
\label{theorem2}
If $\SSV=\ssV\I$, $\SSY=\ssY\I$ and $\B^T\B=\Q\LL\Q^T$, then 
\begin{align}
	\vv|\TT,\yy\sim \mathcal{N}(\mm_{\vv|\TT,\yy},\SSS_{\vv|\TT,\yy}),\notag
\end{align}
where
\begin{align}
	\mm_{\vv|\TT,\yy}&=\mm+\frac{\ssV}{\ssY}\left(\B^T-\frac{1}{\ssY}\Q\LL\left(\frac{1}{\ssV}+\frac{\LL}{\ssY}\right)^{-1}\Q^T\B^T\right)\zz,\notag\\
	\SSS_{\vv|\TT,\yy}&=\Q\left(\frac{1}{\ssV}\I+\frac{1}{\ssY}\LL\right)^{-1}\Q^T.\notag
\end{align}
\end{theorem}

In Theorem \ref{theorem2}, from $\zz$, we can apply matrix multiplications from right to left to get $\mm_{\vv|\TT,\yy}$. The whole computation takes $\OO(D^2+NLd)$. To generate normal samples a Cholseky factorization for $\SSS_{\vv|\TT,\yy}$ is required. But $\left(\frac{1}{\ssV}\I+\frac{1}{\ssY}\LL\right)^{-1}$ is diagonal, so it can be obtained in $\OO(D^2)$ time as well. Now we prove Theorem \ref{theorem2}.
\begin{proof}

$\mm_{\vv|\TT,\yy}$ and $\SSS_{\vv|\TT,\yy}$ can both be derived algebraically. 
\begin{align}
	\mm_{\vv|\TT,\yy}&=\mm+\M\zz\notag\\
	&=\mm+\SSV\B^T(\B\SSV\B^T+\SSY)^{-1}\zz\notag\\
	&=\mm+\SSV\B^T(\SSYI-\SSYI\B(\SSVI+\B^T\SSYI\B)^{-1}\B^T\SSYI)\zz\notag\\
	&=\mm+\frac{\ssV}{\ssY}(\B^T-\frac{1}{\ssY}\B^T\B(\SSVI+\B^T\SSYI\B)^{-1}\B^T)\zz\notag\\
	&=\mm+\frac{\ssV}{\ssY}(\B^T-\frac{1}{\ssY}\Q\LL\Q^T\Q\left(\frac{1}{\ssV}\I+\frac{1}{\ssY}\LL\right)^{-1}\Q^T\B^T)\zz\notag\\
	&=\mm+\frac{\ssV}{\ssY}(\B^T-\frac{1}{\ssY}\Q\LL\left(\frac{1}{\ssV}\I+\frac{1}{\ssY}\LL\right)^{-1}\Q^T\B^T)\zz.\notag
\end{align}
\begin{align}
	\SSS_{\vv|\TT,\yy}&=(\I-\M\B)\SSV\notag\\
	&=(\I-\SSV\B^T(\B\SSV\B^T+\SSY)^{-1}\B)\SSV\notag\\
	&=(\SSVI+\B^T\SSYI\B)^{-1}\notag\\
	&=\Q\left(\frac{1}{\ssV}\I+\frac{1}{\ssY}\LL\right)^{-1}\Q^T.\notag
\end{align}
\end{proof}
\newpage
\section{\camread{Additional experimental results}}

\begin{figure}[htbp]
\centering
\includegraphics[width=0.95\linewidth]{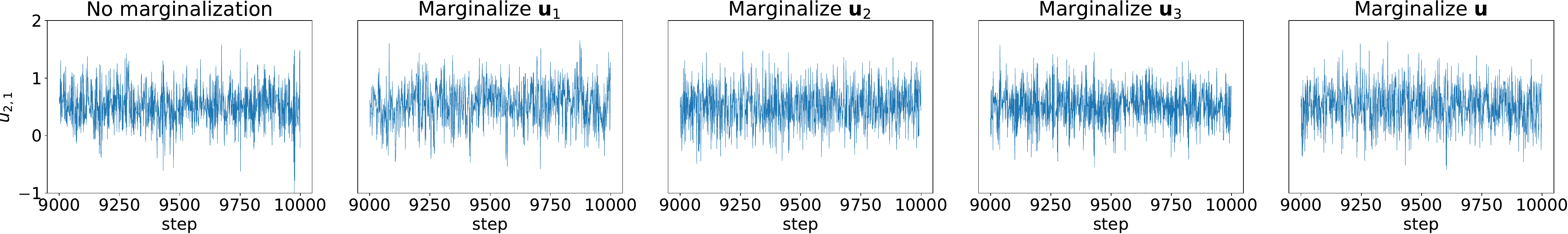}
\hspace*{-4pt}\includegraphics[width=0.959\linewidth]{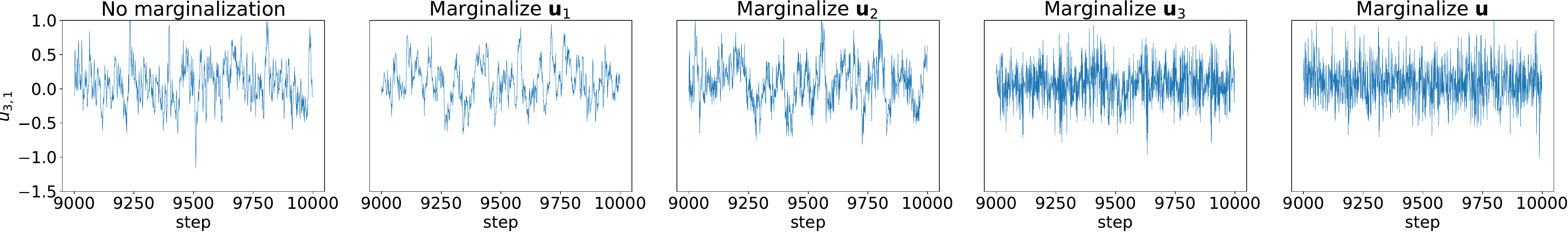}
\includegraphics[width=0.95\linewidth]{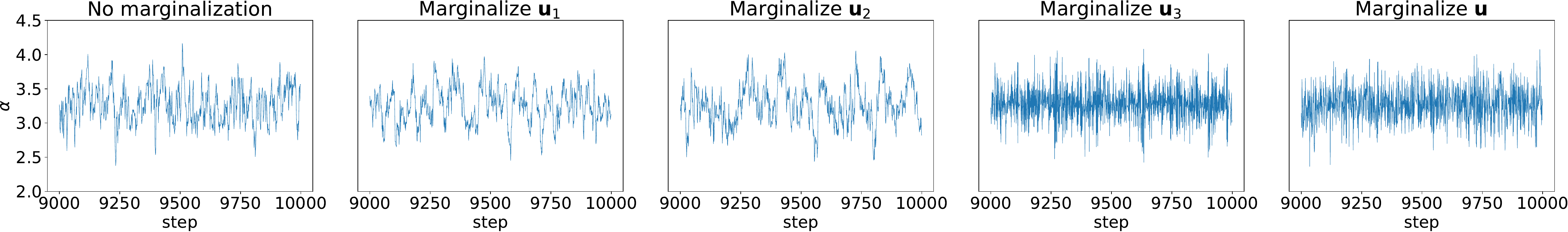}

\caption{Trace plots for $u_{2,1}$, $u_{3,1}$ and $\alpha$ of an interval of 1,000 sampling steps after warmup on the ETH instructor evaluation model, using the same data as Figure \ref{fig:inst_ess} in the paper.}
\label{fig:inst_trace}
\end{figure}
\begin{table}[htbp]
\centering
\caption{Number of parameters (out of 4117) whose $\hat{R}$ exceed a threshold for 1,000 samples from HMC, with or without marginalization. Mean and standard deviation over 5 independent runs are reported. }
\vspace{10pt}
\small
\begin{tabular}{cccccc}
	\toprule
	Threshold&No marginalization&Marginalize $\uu_1$&Marginalize $\uu_2$&Marginalize $\uu_3$&Marginalize $\uu$\\
	\midrule
	$>1.01$&186.80 (37.26)&295.60 (134.57)&11.80 (6.05)&59.80 (37.35)&5.20 (1.72)\\
	$>1.02$&99.40 (18.91)&153.20 (99.90)&6.20 (7.19)&9.80 (10.48)&0.00 (0.00)\\
	$>1.05$&13.40 (6.83)&54.00 (51.99)&0.00 (0.00)&0.00 (0.00)&0.00 (0.00)\\
	$>1.10$&0.00 (0.00)&23.20 (24.51)&0.00 (0.00)&0.00 (0.00)&0.00 (0.00)\\
	\bottomrule
\end{tabular}
\label{tab:inst_rhat}
\end{table}

\begin{table}[!htbp]
\centering
\caption{Divergence (mean and standard deviation) out of 10,000 samples with different strategies on the grouseticks model across 5 random seeds under different target probabilities. We use M1 to represent marginalizing $\uu_1$, M2 to represent marginalizing $\uu_2$, R1 to represent reparameterizing $\uu_1$, R2 to represent reparameterizing $\uu_2$.}
\vspace{10pt}
\small
\begin{tabular}{ccccccc}
	\toprule
	Transformation&Number of divergence \\
	\midrule
	No marginalization&42.60 (25.76)\\
	M1&14.60 (13.85)\\
	M2&0.00 (0.00)\\
	M1, R2&22.60 (18.13)\\
	M2, R1&0.00 (0.00)\\
	R1, R2&431.60 (507.37)\\
	\bottomrule
\end{tabular}
\label{tab:grouse_divergence}
\end{table}

\newpage

\section{Models and example probabilistic programs}
\label{sec:models}
We provide the details of the nine cognitive science datasets and their corresponding models and probabilistic programs. We follow \cite{nicenboim2021introduction} and use maximal models with correlated varying intercept and slopes for each of the datasets. The model for the pupil dataset is described in Section \ref{sec:background}. 
\subsection{Agreement attraction in comprehension}
The dataset (dillonE1) studies the effect of the agreement attraction phenomenon when reading a noun with the auxiliary verb \cite{dillon2013contrasting}. The predictor is
\begin{align}
\log(y_i)=\alpha+u_{1,g_{1,i},1} + u_{2,g_{2,i},1} + t_i(\beta + u_{1,g_{1,i},2} + u_{2,g_{2,i},2})+\epsilon, \epsilon\sim\mathcal{N}(0,\sigma^2).\notag
\end{align}
Each experiment result $y_i$ is from subject $g_{1,i}$ on sentence $g_{2,i}$, with $t_i$ being the interference level ($t_i\in\{0,1\}$). Bayesian hierarchical modeling assigns prior to the variables.
\begin{gather}
\TTT_1\sim\mathcal{N}^{+}(\zero,\text{diag}(5^2,5^2)),\ \LLL_1\sim \LKJ(2,1),\ \TTT_2\sim\mathcal{N}^{+}(\zero,\text{diag}(5^2,5^2)),\ \LLL_2\sim \LKJ(2,1),\notag\\
\alpha\sim\mathcal{N}(0, 10^2),\ \beta\sim\mathcal{N}(0, 5^2),\ \sigma\sim\mathcal{N}^{+}(0, 5^2),\ \uu_{1,j}\sim\mathcal{N}(\zero, \TTT_1\LLL_1\LLL_1^T\TTT_1),\ \uu_{2,k}\sim\mathcal{N}(\zero, \TTT_2\LLL_2\LLL_2^T\TTT_2).\notag
\end{gather}
The probabilistic program in NumPyro is then
\begin{minted}[
frame=single,
fontsize=\tiny,
]
{python}
def model(n_sub, n_item, n_obs, g1, g2, treatment, obs):
alpha = numpyro.sample('alpha', dist.Normal(0, 10))
beta = numpyro.sample('beta', dist.Normal(0, 5))
sigma = numpyro.sample('sigma', dist.HalfNormal(5))
sigma_u = numpyro.sample('sigma_u', dist.LKJCholesky(2))
tau_u = numpyro.sample('tau_u', dist.HalfNormal(5), sample_shape=(2, ))
sigma_v = numpyro.sample('sigma_v', dist.LKJCholesky(2))
tau_v = numpyro.sample('tau_v', dist.HalfNormal(5), sample_shape=(2, ))
s_u = jnp.matmul(jnp.diag(tau_u), sigma_u)
s_v = jnp.matmul(jnp.diag(tau_v), sigma_v)
u = numpyro.sample('u', dist.MultivariateNormal(jnp.zeros((2,)), scale_tril=s_u), sample_shape=(n_sub,))
v = numpyro.sample('v', dist.MultivariateNormal(jnp.zeros((2,)),scale_tril=s_v), sample_shape=(n_item,))
numpyro.sample('y', dist.LogNormal(alpha + u[g1][...,0] + v[g2][...,0] + 
treatment * (beta + u[g1][...,1] + v[g2][...,1]), sigma), obs=obs)	
\end{minted}
We use \texttt{u} and \texttt{v} in the codes to represent the two random effects. The probabilistic program with marginalization is similar. Suppose we marginalize \texttt{u}, our probabilistic program becomes
\begin{minted}[
frame=single,
fontsize=\tiny,
]
{python}
def model(n_sub, n_item, n_obs, g1, g2, treatment, obs):
alpha = numpyro.sample('alpha', dist.Normal(0, 10))
beta = numpyro.sample('beta', dist.Normal(0, 5))
sigma = numpyro.sample('sigma', dist.HalfNormal(5))
sigma_u = numpyro.sample('sigma_u', dist.LKJCholesky(2))
tau_u = numpyro.sample('tau_u', dist.HalfNormal(5), sample_shape=(2, ))
sigma_v = numpyro.sample('sigma_v', dist.LKJCholesky(2))
tau_v = numpyro.sample('tau_v', dist.HalfNormal(5), sample_shape=(2, ))
s_u = jnp.matmul(jnp.diag(tau_u), sigma_u)
s_v = jnp.matmul(jnp.diag(tau_v), sigma_v)
u = jnp.zeros((n_sub, 2))
v = numpyro.sample('v', dist.MultivariateNormal(jnp.zeros((2,)),scale_tril=s_v), sample_shape=(n_item,))
numpyro.sample('y', MarginalizedMultivariateLogNormalGroupCoeff(alpha + u[g1][...,0] + v[g2][...,0] + 
treatment * (beta + u[g1][...,1] + v[g2][...,1]), s_u, sigma, g1, treatment, n_sub, n_obs, u), obs=obs)

\end{minted}
To marginalize \texttt{v}, the probabilistic program is
\begin{minted}[
frame=single,
fontsize=\tiny,
]
{python}
def model(n_sub, n_item, n_obs, g1, g2, treatment, obs):
alpha = numpyro.sample('alpha', dist.Normal(0, 10))
beta = numpyro.sample('beta', dist.Normal(0, 5))
sigma = numpyro.sample('sigma', dist.HalfNormal(5))
sigma_u = numpyro.sample('sigma_u', dist.LKJCholesky(2))
tau_u = numpyro.sample('tau_u', dist.HalfNormal(5), sample_shape=(2, ))
sigma_v = numpyro.sample('sigma_v', dist.LKJCholesky(2))
tau_v = numpyro.sample('tau_v', dist.HalfNormal(5), sample_shape=(2, ))
s_u = jnp.matmul(jnp.diag(tau_u), sigma_u)
s_v = jnp.matmul(jnp.diag(tau_v), sigma_v)
v = jnp.zeros((n_item, 2))
u = numpyro.sample('u', dist.MultivariateNormal(jnp.zeros((2,)), scale_tril=s_u), sample_shape=(n_sub,))
numpyro.sample('y', MarginalizedMultivariateLogNormalGroupCoeff(alpha + u[g1][...,0] + v[g2][...,0] + 
treatment * (beta + u[g1][...,1] + v[g2][...,1]), s_v, sigma, g2, treatment, n_item, n_obs, v), obs=obs)
\end{minted} 
The probabilistic programs for the other models will be similar and we omit them for simplicity.
\subsection{English and Dutch Grammaticality illusion}
The datasets (english \cite{vasishth2010short}, dutch \cite{frank2016cross}) study the VP-forgetting hypothesis \cite{gibson1999memory} for different languages. They use the same predictor and priors. The predictor is
\begin{align}
y_i=\alpha+u_{1,g_{1,i},1} + u_{2,g_{2,i},1} + t_i(\beta + u_{1,g_{1,i},2} + u_{2,g_{2,i},2})+\epsilon, \epsilon\sim\mathcal{N}(0,\sigma^2),\notag
\end{align}
where $t_i$ is the treatment variable and $t_i\in\{-1,1\}$.
And the prior is
\begin{gather}
\TTT_1\sim\mathcal{N}^{+}(\zero,\text{diag}(1^2,1^2)),\ \LLL_1\sim \LKJ(2,1),\ \TTT_2\sim\mathcal{N}^{+}(\zero,\text{diag}(1^2,1^2)),\ \LLL_2\sim \LKJ(2,1),\notag\\
\alpha\sim\mathcal{N}(0, 10^2),\ \beta\sim\mathcal{N}(0, 5^2),\ \sigma\sim\mathcal{N}^{+}(0, 5^2),\ \uu_{1,j}\sim\mathcal{N}(\zero, \TTT_1\LLL_1\LLL_1^T\TTT_1),\ \uu_{2,k}\sim\mathcal{N}(\zero, \TTT_2\LLL_2\LLL_2^T\TTT_2).\notag
\end{gather}
\subsection{Electrophysiological responses with N400 effect}
In the study of language, the electroencephalography (EGG) responses with N400 effect is studied \cite{nicenboim2021introduction}. Experimental results of subjects from the Edinburgh lab are collected \cite{nieuwland2018large}. 
The predictor is
\begin{align}
y_i=\alpha+u_{1,g_{1,i},1} + u_{2,g_{2,i},1} + t_i(\beta + u_{1,g_{1,i},2} + u_{2,g_{2,i},2})+\epsilon, \epsilon\sim\mathcal{N}(0,\sigma^2),\notag
\end{align}
where $t_i$ is the treatment variable and $t_i\in[0,1]$.
And the prior is
\begin{gather}
\TTT_1\sim\mathcal{N}^{+}(\zero,\text{diag}(20^2,20^2)),\ \LLL_1\sim \LKJ(2,1),\ \TTT_2\sim\mathcal{N}^{+}(\zero,\text{diag}(20^2,20^2)),\ \LLL_2\sim \LKJ(2,1),\notag\\
\alpha\sim\mathcal{N}(0, 10^2),\ \beta\sim\mathcal{N}(0, 10^2),\ \sigma\sim\mathcal{N}^{+}(0, 50^2),\ \uu_{1,j}\sim\mathcal{N}(\zero, \TTT_1\LLL_1\LLL_1^T\TTT_1),\ \uu_{2,k}\sim\mathcal{N}(\zero, \TTT_2\LLL_2\LLL_2^T\TTT_2).\notag
\end{gather}

\subsection{Subjective and objective relatives}
\citet{grodner2005consequences} (gg05) studies the processing time difference between object relative clause and subject relative clause sentences. The predictor is
\begin{align}
\log(y_i)=\alpha+u_{1,g_{1,i},1} + u_{2,g_{2,i},1} + u_{3,g_{3,i},1} + t_i(\beta + u_{1,g_{1,i},2} + u_{2,g_{2,i},2} + u_{3,g_{3,i},1})+\epsilon, \epsilon\sim\mathcal{N}(0,\sigma^2),\notag
\end{align}
and the treatment variable $t_i\in\{-1,1\}$. The third effect $\uu_3$ is related to different repeats of the experiment and has only two groups. We consider the first two effects for marginalization to match the other experiments. The prior for the variables is
\begin{gather}
\TTT_1\sim\mathcal{N}^{+}(\zero,\text{diag}(5^2,5^2)),\ \TTT_2\sim\mathcal{N}^{+}(\zero,\text{diag}(5^2,5^2)),\ \TTT_3\sim\mathcal{N}^{+}(\zero,\text{diag}(5^2,5^2)),\notag\\
\LLL_1\sim \LKJ(2,1),\ \LLL_2\sim \LKJ(2,1),\ \LLL_3\sim \LKJ(2,1),\notag\\
\alpha\sim\mathcal{N}(0, 10^2),\ \beta\sim\mathcal{N}(0, 5^2),\ \sigma\sim\mathcal{N}^{+}(0, 5^2),\notag\\
\uu_{1,j}\sim\mathcal{N}(\zero, \TTT_1\LLL_1\LLL_1^T\TTT_1),\ \uu_{2,k}\sim\mathcal{N}(\zero, \TTT_2\LLL_2\LLL_2^T\TTT_2),\ \uu_{3,l}\sim\mathcal{N}(\zero, \TTT_3\LLL_3\LLL_3^T\TTT_3).\notag
\end{gather}
\subsection{Relative clause processing in Mandarin Chinese}
The datasets (mandarin \cite{wu2008processing}, mandarin2 \cite{vasishth2013processing}) are collected from experiments to study the effect of relative clause type on reading time of Mandarin Chinese. In our model, the predictor is
\begin{align}
\log(y_i)=\alpha+u_{1,g_{1,i},1} + u_{2,g_{2,i},1} + t_i(\beta + u_{1,g_{1,i},2} + u_{2,g_{2,i},2})+\epsilon, \epsilon\sim\mathcal{N}(0,\sigma^2),\notag
\end{align}
where $t_i$ is the treatment variable and $t_i\in\{-0.5,0.5\}$.
And the prior is
\begin{gather}
\TTT_1\sim\mathcal{N}^{+}(\zero,\text{diag}(5^2,5^2)),\ \LLL_1\sim \LKJ(2,1),\ \TTT_2\sim\mathcal{N}^{+}(\zero,\text{diag}(5^2,5^2)),\ \LLL_2\sim \LKJ(2,1),\notag\\
\alpha\sim\mathcal{N}(0, 10^2),\ \beta\sim\mathcal{N}(0, 5^2),\ \sigma\sim\mathcal{N}^{+}(0, 5^2),\ \uu_{1,j}\sim\mathcal{N}(\zero, \TTT_1\LLL_1\LLL_1^T\TTT_1),\ \uu_{2,k}\sim\mathcal{N}(\zero, \TTT_2\LLL_2\LLL_2^T\TTT_2).\notag
\end{gather}

\subsection{The Stroop effect}
The Stroop effect describes the change of response time between congruent and incongruent stimuli \cite{macleod1991half}. The dataset is from \citet{ebersole2016many}. Different from the other models, the noise scale for each observation is also grouped. In our model, the predictor is
\begin{gather}
\log(y_i)=\alpha+u_{g_{i},1} + t_i(\beta + u_{g_{i},2})+\epsilon,\ \epsilon\sim\mathcal{N}(0,\sigma_i^2),\ \sigma_i=\exp(\sigma_\alpha+s_{g_i,1} + t_i(\sigma_\beta + s_{g_i,2})),\notag
\end{gather}
and the treatment variable is $t_i\in\{-1,1\}$. Priors for the model are
\begin{gather}
\TTT_{\uu}\sim\mathcal{N}^{+}(\zero,\text{diag}(1,1)),\ \LLL_{\uu}\sim \LKJ(2,1),\ \TTT_{\sigma}\sim\mathcal{N}^{+}(\zero,\text{diag}(1,1)),\ \LLL_{\sigma}\sim \LKJ(2,1),\notag\\
\alpha\sim\mathcal{N}(6, 1.5^2),\ \beta\sim\mathcal{N}(0, 0.01^2),\ \sigma_\alpha\sim\mathcal{N}(0, 1),\ \sigma_\beta\sim\mathcal{N}(0,1),\notag\\
\uu_{j}\sim\mathcal{N}(\zero, \TTT_{\uu}\LLL_{\uu}\LLL_{\uu}^T\TTT_{\uu}),\ \mathbf{s}_{j}\sim\mathcal{N}(\zero, \TTT_{\sigma}\LLL_{\sigma}\LLL_{\sigma}^T\TTT_{\sigma}).\ \notag
\end{gather}

\end{document}